\documentclass[ejs,dvips]{imsart}

\RequirePackage{natbib}

\usepackage{amsthm,amsmath,natbib}
\RequirePackage[colorlinks,citecolor=blue,urlcolor=blue]{hyperref}

\usepackage{graphicx}
\usepackage{amsfonts}

\newtheorem{thm}{Theorem}

\newtheorem{asm}{Assumption}
\newtheorem{lem}{Lemma}

\begin{document}

\begin{frontmatter}

\title{Theoretical Properties of the Overlapping
  Groups Lasso}
\runtitle{Theory of Overlapping Group Lasso}

\begin{aug}
\author{\fnms{Daniel} \snm{Percival}\thanksref{t1,t2,t3}\ead[label=e1]{dperciva@andrew.cmu.edu}}
\address{Carnegie Mellon University\\
Department of Statistics\\
Pittsburgh, PA 15213 USA\\
\printead{e1}\\
}
\runauthor{Percival}

\thankstext{t1}{This work is part of the author's Ph.D thesis.}
\thankstext{t2}{This work was funded by the National Institutes of Health grant
MH057881 and National Science Foundation grant DMS-0943577.  The author was also partially supported by National Institutes of Health SBIR grant 7R44GM074313-04 at Insilicos LLC.}
\thankstext{t3}{The author would like to thank Larry Wasserman and Aarti Singh for helpful comments and discussions, as well as the two anonymous referees whose suggestions helped greatly improve the contents of this paper.}

\affiliation{Carnegie Mellon University, Department of Statistics}

\end{aug}

\begin{abstract}
We present two sets of theoretical results on the
grouped lasso with overlap due to~\cite{Laurent-graph-lasso} in the linear
regression setting. This method jointly selects
predictors in sparse regression, allowing for complex
structured sparsity over the predictors encoded as a set of
groups. This flexible framework suggests that arbitrarily complex structures
can be encoded with an intricate set of groups. Our results show that
this strategy results in unexpected theoretical consequences for
the procedure. In particular, we
give two sets of results: (1) finite sample bounds on prediction and
estimation, and (2) asymptotic distribution and selection. Both sets
of results demonstrate negative consequences from choosing an
increasingly complex set of groups for the procedure, as well for when the set of groups cannot recover the true sparsity
pattern. Additionally, these results demonstrate the differences and
similarities between the the grouped lasso procedure with and without
overlapping groups. Our analysis shows that while the procedure enjoys advantages over the standard lasso, the set of groups must be
chosen with caution --- an overly complex set of groups will damage the
analysis.

\end{abstract}



\end{frontmatter}

\section{Introduction}

In this paper, we consider the linear regression model: $\textbf{y} =
\textbf{X}\beta^0 + {\epsilon}$, where $\textbf{X}$ is an
$n \times p$ real valued data matrix, $\textbf{y} \in \mathbb{R}^n$ is
a vector of responses, $\beta^0 \in \mathbb{R}^p$ is a vector of
linear weights, and ${\epsilon}$ is an error vector.  Much
work focuses on estimating a {\em sparse} $\widehat{\beta}$, where
many of the entries are equal to zero, effectively excluding many of
the dimensions of $\textbf{X}$ --- the candidate predictors --- from
the model. Recent work adds the notion of {\em   structure} to this
setting.  That is, we desire the set of nonzero entries in
$\widehat{\beta}$ to follow some predefined structure over the
candidate predictors. There are now many methods tailored to a diverse
collection of structures, including hierarchical structures, group
structures, and graph derived structures: see~\cite{Bach_levelsets, Bach10hierarchical,
  Bach10structuredsparse, Huang:2009:LSS:1553374.1553429, Jenatton_Audibert_Bach_2009,
  Jenatton09structuredsparse, pengMasterPredictors,
  percivalStructureSparse, xingTreeGuided, zhoaCAP} for examples. 

One such structured sparse method is the grouped lasso of~\cite{Yuan06modelselection}, which extends
the familiar $\ell_1$ penalization to a grouped $\ell_1$ norm. In
particular, the grouped lasso allows for groups of predictors to enter
the model together, a useful property in settings such as ANOVA or multi-task regression.  For example, we can encode
 a factor predictor with $m$ levels as $m-1$
 indicator variables in $\textbf{X}$. When we build a sparse
 regression model, we might prefer to select none or all 
 of this group of $m-1$ variables, 
but not any other subset.  The grouped lasso enables this type of grouped selection.
\cite{bach-group-lasso, chesnut-grouped, zhang_benefits, Lounici_takingadvantage,
  nardi-grouped-lasso} give theoretical results for this procedure, including oracle inequalities and asymptotic distributions. In particular, they showed that for some problems the grouped lasso outperforms the ordinary lasso.

However,
the grouped $\ell_1$ norm of the grouped lasso is limited in that it
only allows groups that partition the set of candidate
predictors. This restricts the complexity and types of structures over
the candidate predictors that can be encoded in the groups.  For
example, we could represent the potential structures of $\beta^0$ as a a graph over $p$ nodes, where each node represents
a candidate predictor.  We might then seek to build a sparse model where the selected predictors
correspond to a subgraph, such as a neighborhood or clique, of this graph. This structure can be encoded, for example, as a series of overlapping neighborhoods, such as 4-cycles in a 2-dimensional lattice graph. The grouped $\ell_1$ norm does not allow for such a set of groups.

The grouped lasso with overlap of~\cite{Laurent-graph-lasso} is one solution to this
problem (see also the CAP penalty of~\cite{zhoaCAP}, as well as other
group based procedures in~\cite{Bach10structuredsparse,
  Jenatton_Audibert_Bach_2009}). Using an extension of the 
grouped norm of the grouped lasso, this procedure allows for complex,
overlapping group structures.  Given a collection of subsets of the
set of candidate predictors, the procedure recovers nonzero patterns
equal to a union of some subset of this collection.  This property can encode many complex
structures over the candidate predictors, and thus within the
resulting sparsity patterns of the estimated
coefficients. While~\cite{Laurent-graph-lasso} gave some initial
theoretical results on this procedure, including a consistency result,
many theoretical questions were left open.  In particular, the impact
on the predictive and estimation  performance of the procedure of
increasingly complex sets of groups remained unanswered. The overlapping nature of the groups allows the
possibility for an arbitrarily large set of  groups to encode complex
structures, or many possible structures simultaneously. If we suppose
that there is no consequence to increasing the number and complexity
of the groups, then we can freely run the procedure under many
structural conditions simultaneously. 

The concluding remarks of~\cite{zhang_benefits} indicate that the grouped lasso does not perform well with overlapping groups. The goal of this paper is to expose exactly how introducing the possibility of overlapping groups impacts the grouped lasso. Towards this goal, we demonstrate some theoretical properties of the
overlapping grouped lasso, with a focus on the consequence of the
number and complexity of the groups of predictors.  We give a finite sample and an asymptotic
result.  In particular, we make the following contributions:
\begin{enumerate}
\item We show that both the finite sample and asymptotic performance
  of the overlapping grouped lasso suffers as the number and complexity of the groups grows.
\item In the finite sample case, we show that the 
  assumptions on the design matrix $\textbf{X}$ become more
  restrictive as the   complexity of the groups grows.
\item In our asymptotic analysis, we introduce the adaptive overlapping grouped
  lasso, and give an   adaptive weighting scheme with asymptotic
  selection guarantees   similar to the adaptive lasso of~\cite{Zou_2006}
  (see also the adaptive   grouped lasso results in~\cite{nardi-grouped-lasso}). 
\end{enumerate}

Overall, we conclude that the overlapping grouped lasso enjoys many of
the same theoretical guarantees as the grouped lasso, provided that
the set of groups are not too complex or large.  We therefore recommend
that the procedure should be used  with a set of groups that is not
overly complex, or contains a nested structure.

The paper is organized as follows: we first introduce notation
for the overlapping grouped lasso.  We also reproduce some basic theoretical
properties of the procedure and the associated overlapping grouped norm.
We next give our finite sample results, and then our asymptotic
results.  We then present a simulation study to support our theoretical results. Proofs of the main results along with supporting lemmas appear in the appendix. 

\section{Notation}

We adopt a combination of the notation
of~\cite{Laurent-graph-lasso} 
and~\cite{Lounici_takingadvantage}. Recall our basic setting, the linear model: 
\begin{align}
\label{mod}
\textbf{y} = \textbf{X}\beta^0 + \epsilon.
\end{align}
Here, $\textbf{X}$ is an $n \times p$ data matrix, $\textbf{y} \in
\mathbb{R}^n$ is the  response, $\beta^0 \in \mathbb{R}^p$ is a
vector of true linear coefficients, and $\epsilon$ is a stochastic
error term.  Our goal is to estimate a sparse $\widehat{\beta}$, such
that the nonzero entries follow some structure which we assume to be
known a priori. In particular, we consider structures defined in terms
of groups of predictors, which we define as subsets of the set of  candidate predictors indices:
$\mathcal{I} = \{1,2,\ldots,p\}$.  We denote a collection of groups as
$\mathcal{G}$ with elements $g$ such that each $g \subseteq \mathcal{I}$.  Let $|\mathcal{G}| =
M$, and assume $\displaystyle\bigcup_{g \in   \mathcal{G}} g = \mathcal{I}$.  For
coefficient vectors $\beta$, we define $\beta_g \in \mathbb{R}^{|g|}$
as the sub vector consisting of the entires corresponding to the
indices in $g$.  Define the support of a vector as: 
\begin{align}
\mbox{supp}(\beta) = \{i: \beta_i \neq 0\} \subseteq \mathcal{I}.   
\end{align}

We now give a
framework, proposed by~\cite{Laurent-graph-lasso}, to measure the structured sparsity of vectors in
$\mathbb{R}^p$.  We define the following
convention: for vectors denoted $v_g \in \mathbb{R}^p$ we have that 
$\mbox{supp}(v_g) \subseteq g$.  We define a
decomposition of $\beta \in \mathbb{R}^p$ with respect to
$\mathcal{G}$ as:
\begin{align}
\label{theDecomposer}
 \mathcal{V}_\mathcal{G}(\beta) = \left\{ v_g: g \in\mathcal{G} \right\} \mbox{
   such that } \sum_{g \in \mathcal{G}} v_g = \beta .
\end{align}
That is,  each decomposition in $\mathcal{V}_\mathcal{G}(\beta)$ is a
collection of $M$
vectors in $\mathbb{R}^p$ each satisfying $\mbox{supp}(v_g) \subseteq
g$ for a different $g \in \mathcal{G}$. From now on, we suppress
the $\mathcal{G}$ in the notation for decompositions and write $\mathcal{V}(\beta)$.  
$\mathcal{V}(\beta)$ is not unique in general.  We define the following
norms: 
\begin{align}
\label{norm}
||\beta||_{2,p,\mathcal{G}} &= \min_{\mathcal{V}(\beta)}\left( \sum_{g
    \in \mathcal{G}} \left( \sum_{i \in g} v_{g,i}^2 \right)^{p/2}
\right)^{1/p},\\ 
&= \min_{\mathcal{V}(\beta)}\left( \sum_{g \in \mathcal{G}} ||v_g||^p
\right)^{1/p}; \label{overlapNorm}\\
||\beta||_{2,\infty,\mathcal{G}} &= \min_{\mathcal{V}(\beta)} \max_{g \in
  \mathcal{G}}  ||v_g||. 
\end{align}
Here, $||\cdot||$ denotes the Euclidean or $\ell_2$ norm.  The above
two equations are norms by arguments presented
in~\cite{Laurent-graph-lasso}.  Note that the notation
$\min_{\mathcal{V}(\beta)}$ indicates the minimum over all possible
decompositions.  Note that the decomposition that minimizes
these norms is not necessarily unique, as we state in the following
lemma.

\begin{lem}
\label{lemmaForUniqueNess}
(Corollary 1 from~\cite{Laurent-graph-lasso}) For any collections $\{v_g\}$, $\{v_g'\}$ minimizing the
norm~\ref{overlapNorm}, we have, $\forall g \in \mathcal{G}$:
\begin{align}
||v_g|| \times ||v_g'|| = 0 \mbox{ or } \frac{v_g}{||v_g||} = \frac{v_g'}{||v_g'||}.
\end{align}
\end{lem}

The above lemma implies that in some cases the collection of groups used in the decomposition --- that is, $\{g \in \mathcal{G} \mbox{ s.t. } v_g \neq 0$\} --- is not unique.

Finally, for $J \subset \mathcal{G}$ we write $\beta_{J} = \sum_{g \in G}
v_g 1_{g \in J}$, note that $J \subseteq \{1,2,\ldots,M\}$.  Let $J_v(\beta) = \{g: v_g \neq 0\}$, and $M_v(\beta) =
|J_v(\beta)|$. Thus, $J_v(\beta)$ is the set of groups used to
decompose $\beta$ for a particular decomposition. $M_v(\beta)$ is thus
a measure of the structured sparsity of $\beta$ with respect to a
particular decomposition. Let $M(\beta) = \min_v M_v(\beta)$, where this minimum is taken over  the set of decompositions
minimizing the norm~\ref{overlapNorm}.  Thus, $M(\beta)$
measures the overall structured sparsity of $\beta$, with respect to
the groups $\mathcal{G}$.

Here is a simple example to illustrate the setting.  Let $p=3$,
and consider the following groups:
\begin{align}
\mathcal{G} = \{\{1,2\}, \{2,3\} \}.
\end{align}
For any $\alpha \in \mathbb{R}$, we have the following possible
decomposition of $\beta = [a,b,c]$:
\begin{align}
\mathcal{V}(\beta) &= \{v_{\{1,2\}}, v_{\{2,3\}}\},\\
v_{\{1,2\}} &= [a,\alpha b, 0],\\
v_{\{2,3\}} &= [0,(1-\alpha) b, c].
\end{align}
Thus, the norm from Equation~\ref{overlapNorm} can be expressed as:
\begin{align}
||\beta||_{2,1,\mathcal{G}} = \min_\alpha \left( \sqrt{a^2 + (\alpha b
 )^2} + \sqrt{ ((1-\alpha) b)^2 + c^2} \right).
\end{align}
Finally, it is clear that for $a, c \neq 0$, $M(\beta) = 2$, and $M(\beta) = 1$ otherwise.

\section{Overlapping Grouped Lasso}

Recall our goal, under the model of Equation~\ref{mod}, we estimate the target $\beta^0$ with a sparse $\widehat{\beta}$ -- that is, many entires of
$\widehat{\beta}$ are set to zero.  Additionally, we know these nonzero entries
occur in a structured pattern, as given by $\mathcal{G}$.  We evaluate the fit with the usual quadratic loss:
\begin{align}
\ell(\beta) &=  \frac{1}{n} ||\textbf{y} - \textbf{X}\beta||^2.
\end{align}
The overlapping grouped lasso solves the following optimization problem:
\begin{align}
\label{opt}
\widehat{\beta}  = \underset{\boldsymbol{\beta} \in
  \mathbb{R}^p}{\operatorname{argmin}} \left( \ell(\beta) + 2\lambda||\beta||_{2,1,\mathcal{G}}\right).
\end{align}
Here $\lambda >0$ is a tuning parameter controlling the amount of regularization. If the elements of $\mathcal{G}$ are restricted to be pairwise disjoint, then the norm $||\cdot||_{2,1,\mathcal{G}}$ reduces to the grouped $\ell_1$ norm. We then
recover the original formulation of the grouped lasso.  In the special
case where the groups are all singletons: $\mathcal{G}
= \{\{i\}: i \in \mathcal{I}\}$, we recover the familiar lasso~\cite{tibshiraniLasso}.  If we allow
$\mathcal{G}$ to be any collection, allowing for the possibility of
overlap between groups, then the minimum over
$\mathcal{V}(\beta)$ in the norm now plays a role since the decomposition of
$\beta$ is no longer unique in general.  This setting gives us the overlapping grouped
lasso.  For each of these problems, the key fact is that the support
of $\widehat{\beta}$ will be a union of members of a subset of $\mathcal{G}$.
Finally, we also introduce the the adaptive overlapping
grouped lasso:  
\begin{align}
\label{adopt}
\widehat{\beta}  = \underset{\boldsymbol{\beta} \in
  \mathbb{R}^p}{\operatorname{argmin}} \left(  {\ell}(\beta) + 2\lambda \min_{\mathcal{V}(\beta)}\sum_{g \in \mathcal{G}} \lambda_g||v_g|| \right).
\end{align}
As previous work and theory has suggested (\cite{nardi-grouped-lasso, Zou_2006}), the choice of
weights: $\lambda_g = 1/||\beta^{{OLS}}_g||^\gamma$, where
$\beta^{OLS} =
(\textbf{X}^T\textbf{X})^{-1}\textbf{X}^{T}\textbf{y}$, and $\gamma>0$, gives good asymptotic guarantees.  In Section~\ref{asymSection}, we
 show that a different choice is needed in our setting to give similar asymptotic guarantees.



Finally, as noted by~\cite{Laurent-graph-lasso}, the
overlapping grouped lasso method is simple to implement.
In the case where $\mathcal{G}$ consists of non-overlapping groups,
there are several efficient algorithms available.  In the overlapping
case, no new specialized algorithm is required.  Write $\textbf{X}_g$ as
the sub-matrix of $\textbf{X}$ with only the   columns of $\textbf{X}$
indexed by the elements of $g$. Now define $\widetilde{\textbf{X}} =
[\textbf{X}_g]_{g \in   \mathcal{G}}$ --- a $n \times \sum_g|g|$
matrix of the concatenation of the columns of $\textbf{X}$
corresponding to each group in $\mathcal{G}$.  We then can solve the
optimization problem with with a new, non-overlapping, set of groups $\mathcal{G}$ defined on the
appropriate columns of $\widetilde{\textbf{X}}$.  Since $\mathcal{G}$ is now a
non-overlapping set of groups for $\widetilde{\textbf{X}}$, we can simply apply existing algorithms for the grouped lasso.

\section{Finite Sample Bounds}

We now give a sparsity oracle inequality for the overlapping grouped
lasso.  This finite sample result is an extension of a result on
multitask regression due to~\cite{Lounici_takingadvantage}, which is
in turn  built on results
from~\cite{Bickel_simultaneousanalysis}.  We first state and discuss our main assumption, which is an adaptation of the
restricted eigenvalue condition of~\cite{Bickel_simultaneousanalysis} to the overlapping grouped
lasso seetting.


\begin{asm}
\label{asmREC}
Suppose $1 \leq s \leq M = |\mathcal{G}|$.  Then there
exists $\kappa(s) > 0$ such that:
\begin{align}
\kappa(s) &\leq \min \left\{ \frac{\sqrt{\Delta^TX^TX\Delta}}{\sqrt{n} \sum_{g \in
      J}||v_g^\Delta||}:  J \subseteq \mathcal{G}; J \in
  \mathcal{J}(s) \right\},\\
\mathcal{J}(s) &:=  \left\{J \subseteq \mathcal{G}; |J| \leq s; \Delta \in \mathbb{R}^p
  \backslash \textbf{0};  \ \mathcal{V}(\Delta) = \{v_g^\Delta\}  \mbox{  s.t. } \sum_{g \in J^c}||v_g^\Delta|| \leq 3 \sum_{g \in J}||v_g^\Delta||\right\}.
\end{align}
Here $J^c = \{g: g\in \mathcal{G}, g \notin J\}$, and $\mathcal{V}(\Delta) = \{v_g^\Delta\}$  denotes the decomposition minimizing the norm $||\Delta||_{2,1,\mathcal{G}}$.
\end{asm}

In the subsequent results, the integer $s$ measures the structured sparsity of the target.  There are
two key differences between this assumption and other restricted
eigenvalue conditions.  First, it relies on norms of the {\em decompositions}
 of vectors, rather than norms of the
vector or appropriate sub-vectors.  Note that the decomposition of $\Delta$ must
be a decomposition minimizing the $||\cdot||_{2,1,\mathcal{G}}$ norm.
As we will discuss later, this condition grows more restrictive as
 $\mathcal{G}$ becomes more complex. The key second difference in the assumption lies in the denominator
  term $\sum_{g \in    J}||v_g^\Delta||$, which appears instead of the
  directly analogous $|| \sum_{g \in J} v_g^\Delta ||$. We know by the
  triangle inequality that $|| \sum_{g \in J} v_g^\Delta || \leq
  \sum_{g \in   J}||v_g^\Delta||$, and so our $\kappa$ is less than or
  equal to a $\kappa'$ obtained under the  analogous assumption.  In the case of non-overlapping groups, this is an equality, and the assumption is identical in this case.
  
We now examine some sufficient conditions for the existence of $\kappa(s)$. Examining the numerator of the main quantity defining $\kappa(s)$, we see that $\sqrt{\Delta^TX^TX\Delta/n} \geq |\rho_X|^{1/2}||\Delta||$, where $\rho_X$ is the minimal eigenvalue of $X^TX/n$. Examining the denominator, we can make the following bounds:
\begin{align}
\sum_{g\in J} ||v_g^\Delta|| &\leq \sum_{g \in \mathcal{G}} ||v_g^\Delta||\\
&\leq \sum_{g \in \mathcal{G}} ||\Delta_g||\\
&\leq ||\Delta||(M\mathcal{G}_{\mbox{overlap}})^{1/2}.
\end{align}
Here, $\mathcal{G}_{\mbox{overlap}} := \max_{j \in \mathcal{I}} \left[\sum_{g \in \mathcal{G}} 1_{j\in g} \right]$ is the maximal number of times a candidate predictor appears in the groups of the collection $\mathcal{G}$. Thus, as long as $X^TX$ has a nonzero minimal eigenvalue, we are guaranteed to find a $\kappa(s)$ of at most $(\rho_X/M\mathcal{G}_{\mbox{overlap}})^{1/2}$. In particular, for $\kappa (s)$ to exist, it is sufficient for $X^TX$ to be positive definite. We now state our main result.


\begin{thm}
\label{finiteProp}
Consider the model in Equation~\ref{mod}.  Suppose $|\mathcal{G}| = M
\geq 2$, and $n \geq 1$.  Assume that the entries of $\epsilon$ are
i.i.d. Gaussian with mean 0 and variance $\sigma^2$.  Let $\textbf{X}$
be normalized so that the the diagonal entries of
$\textbf{X}^T\textbf{X}/n$ are all equal to 1.  Denote $M(\beta^0) \leq
s$ as the maximum number of nonzero groups in decompositions of
$\beta^0$, $\mathcal{V}(\beta^0)$. Let Assumption~\ref{asmREC} hold
with $\kappa = \kappa(s)$.  Let:
\begin{align}
\lambda = \frac{2\sigma \sqrt{\max_g|g|\mathcal{G}_{\mbox{overlap}}}}{\sqrt{n}} \left( 1 + \frac{A\log M}{\sqrt{\max_g|g|}}\right)^{1/2}.
\end{align}
Here, $A>8$.  Define $q = \min_g(\rho_g^{-2})  \min \left(A\sqrt{\min_g |g|}/8,8 \log M \right)$, where $\rho_g$ is the maximal absolute eigenvalue of a Cholesky decomposition of $\textbf{X}_g^T\textbf{X}_g$, where $\textbf{X}_g$ is the sub matrix of $\textbf{X}$ corresponding to the columns indexed by the group $g$.  Then, with probability at least
$1-M^{1-q}$, for any solution $\widehat{\beta}$ to Equation~\ref{opt}, for
all $\beta^0 \in \mathbb{R}^p$, the following inequalities hold:
\begin{align}
\label{p1res1}
\frac{1}{n}||X(\widehat{\beta} - \beta^0)||^2 &\leq
\frac{64\sigma^2 }{\kappa^2 n} \left(\max_g|g| + A\sqrt{\max_g|g|}\log M \right),\\
\label{p1res2}
||\widehat{\beta} - \beta^0 ||_{2,1,\mathcal{G}} &\leq
\frac{32\sigma }{\kappa\sqrt{n}}\left( \max_g|g| + A\sqrt{\max_g|g|}\log M
\right)^{1/2}.
\end{align}
\end{thm}

The proof for this result is given in the appendix~\ref{propositionProofOracle}.  The proof relies on Lemma~\ref{lemmaOracle} given in the appendix~\ref{lemmaFiniteSection}.  We now discuss the result.

\label{commentsFinite}
\begin{enumerate}
\item  {\em As the set of groups grows, the
    finite sample guarantees degrade}. In Proposition~\ref{finiteProp},
    the prediction and estimation bounds both get coarser as the
    number of groups increases. Note that the set of groups can grow not only as the
    dimension of the problem grows, but also if we encode complex
    structures over the predictors using $\mathcal{G}$. Thus, even for a
    problem of fixed dimension $p$, there is a consequence to
    choosing an arbitrarily complex set of groups. To make this result
    clear, let the groups be maximally complex: $\mathcal{G} =
    2^\mathcal{I}$, the power set of the set of predictors. Now,
    as the dimension of the problem grows, the prediction bound grows
    at rate $O(p^{3/2})$, and the estimation bound at rate $O(p)$. If $|\mathcal{G}|$ is instead of the same order as $p$ and the maximum group size is constant, these rates are instead both $O(\log p)$.  This shows that grouped sparsity achieves the tightest upper
    bounds if both
    the maximum group size and the number of groups grow at slower rates than $p$. Note the contrast here to the results of~\cite{Lounici_takingadvantage} in the multi-task setting, where a growing number of tasks benefitted the procedure. Note that in multi-task setting, the number of observations necessarily grows with the number of tasks, contrary to our setting.
\item {\em As the complexity of the groups grows,
    Assumption~\ref{asmREC} becomes more restrictive}. Since $\kappa$ appears in the denominator of both the prediction and estimation bounds, the bounds become less tight as $\kappa$ decreases.
  Consider the condition:
\begin{align}
\sum_{g \in J^c}||v_g^\Delta|| \leq 3 \sum_{g \in J}||v_g^\Delta||.
\end{align}
Recall that $J$ is a cardinality $s$ set of groups. Thus, for fixed $s$,
as the complexity of $\mathcal{G}$ grows, the flexibility of the decompositions grows, and then more vectors $\Delta$ satisfy this
condition.  This makes $\kappa$ decreasing as a function of
$|\mathcal{G}|$. We also recall that when $X^TX$ has a nonzero minimal absolute eigenvalue, we know $\kappa$ is at most $\sqrt{\rho_X/M\mathcal{G}_{\mbox{overlap}}}$. As noted earlier, as the complexity of the groups grows, $\mathcal{G}_{\mbox{overlap}}$ increases as well, leading to a smaller $\kappa$ and in turn inferior prediction and estimation bounds. If $\mathcal{G}_{\mbox{overlap}}$ is on the same order as the number of predictors, then $\kappa (s)$ is of order $1/M$ rather than $1/\sqrt{M}$. This dependence shows that our bounds depend equally on the dimension of the problem $M$ and the group complexity as measured by $\mathcal{G}_{\mbox{overlap}}$. In the case of the lasso or group lasso, $\mathcal{G}_{\mbox{overlap}} = 1$, giving us no dependence on group complexity, as expected.

\item {\em The results show that the procedure enjoys an advantage over non-structured procedures when $\beta^0$ is structured sparse}. For example, in the finite sample case, none of our bounds depended explicitly on the dimension of the problem $p$.  Thus, we can adopt a similar argument to those of~\cite{Lounici_takingadvantage} to show that compared to the lasso, the overlapping grouped lasso gives superior results in the case where $\beta^0$ is structured sparse.  That is, from~\cite{Bickel_simultaneousanalysis}, if we let:
\begin{align}
\lambda = A\sigma \sqrt{\frac{\log p}{n}},
\end{align}
then for $A > 2\sqrt{2}$, we have that with probability at least $1 - (p)^{1 - A^2/8}$:
\begin{align}
\frac{1}{n} || X (\widehat{\beta}_{\mbox{lasso}} - \beta^0 ) ||^2 \leq \frac{16A^2\sigma^2}{\kappa^2n} \log p.
\end{align}
Thus, if $ \sqrt{\max_g |g|} \log M + \max_g |g|$ is of smaller order than $\log p$, the procedure has a predictive advantage.  Since $\kappa$ depends on the structured sparsity of the target, this result holds only for structured sparse targets $\beta^0$ which give sufficiently large values of $\kappa$ under our assumption.

\item {\em In the non-overlapping case, we can recover many results
    available in the literature}. Here we have $\mathcal{G}_{\mbox{overlap}} = 1$. We adjust our assumption to
  match the literature, so
  that the quantity in the minimum is replaced with:
\begin{align}
\frac{\sqrt{\Delta^TX^TX\Delta}}{\sqrt{n} \left|\left|\sum_{g \in
      J}v_g^\Delta\right|\right|}.
\end{align}
Combining this with an application of the Cauchy-Schwarz inequality in the last steps of the proofs of the result,
we can recover the results of~\cite{Lounici_takingadvantage} in the multi-task case. In the case of the grouped lasso,
we can recover the result from~\cite{nardi-grouped-lasso}.  The dependence on the minimal eigenvalues of the Cholesky decomposition of each $\textbf{X}_g^T\textbf{X}_g$ is related to the conditions given in~\cite{zhang_benefits}. In the settings of~\cite{Lounici_takingadvantage}, $\rho_g = 1$ for all $g$.
\item {\em We can show a similar result solely in terms of
    $\max_g |g|$}.  In particular, for:
\begin{align}
\lambda = \frac{2\sigma\sqrt{\mathcal{G}_{\mbox{overlap}}}}{\sqrt{n}} \left( \max_g|g| + A\log M\right)^{1/2},
\end{align}
the same results hold with probability $1 - M^{1-q}$, for $q = \min_g(\rho_g^{-2}) \min
\left(A/8, \frac{8 \log M}{\max_g |g|} \right)$.  This result is a
consequence of a simple adjustment for this choice of $\lambda$ in the proof of
Lemma~\ref{lemmaOracle} from the appendix. This alternate result shows
that as the maximum group size grows, the estimation and prediction
bounds become less tight, and the probability that they hold falls. 
\item {\em The result does not depend on the any uniqueness
    assumptions on the decomposition of $\beta^0$}. The
  consistency result for the overlapping grouped lasso
  in~\cite{Laurent-graph-lasso} assumes that the decomposition of
  $\beta^0$ that minimizes the $||\cdot||_{2,1,\mathcal{G}}$ norm is
  unique. Our result, in contrast, depends only on the maximal
  structured sparsity of such decompositions. Thus, in the case where $\beta^0$ does not have a unique decomposition minimizing the $||\cdot||_{2,1,\mathcal{G}}$ norm, our results still hold. This is a contrast to the asymptotic results of the next section.
\end{enumerate}

%

\section{Asymptotic Results}
\label{asymSection}

In this section, we consider fixed dimension asymptotic for the
adaptive overlapping grouped lasso as described in Equation~\ref{adopt}.
These results extend those on the grouped lasso found
in~\cite{nardi-grouped-lasso} to the case of overlapping groups. 

To begin, define the set of indices of the true linear coefficient vector
$\beta^0$ which are nonzero and zero as the following:  
\begin{align}
H &=  \{i: \beta^0_i \neq 0\},\\
H^c &= \{i: \beta^0_i = 0\}.
\end{align}
Accordingly, we
define $\textbf{X}_H$ as the sub matrix containing the entries with
 column indices in the set $H$.  Similarly, for a $p$-vector $x$, let $x_H$ be the sub vector containing the entries with indices in the set $H$.  Clearly, $H \cup H^c =
\mathcal{I}$.  However, that
$H$ and $H^c$ are not necessarily the union of members of $J(\beta^0)$
and $J(\beta^0)^c$, respectively.  We next define the following three subsets of $\mathcal{G}$ related to $H$ and $H^c$: 
\begin{align}
G_H =& \{g: g \subseteq H\},\\
G_{Hc} =& \{g: g \subseteq H^c\},\\
G_{Ho} =& \{g: |g \cap H| >0; |g \cap H^c| > 0 \}.
\end{align}
These are, respectively, the set of groups in which the indices are all
nonzero in $\beta^0$, all zero in $\beta^0$, and a mix of zero and
nonzero in $\beta^0$.

For this setting, we now make the following assumptions:
\begin{asm}
\label{A:FIXED}
As $n \to \infty$, $\textbf{X}^T\textbf{X} \to \mathcal{M}$, where $\mathcal{M}$ is positive definite.
\end{asm}
\begin{asm}
\label{B:FIXED}
 The entries of the stochastic term $\boldsymbol{\epsilon}$ in
 Equation~\ref{mod} are $\mbox{i.i.d.}$ with finite second moment $\sigma^2$.
\end{asm}
\begin{asm}
\label{correct}
There exists a neighborhood in $\mathbb{R}^p$ around $\beta^0$ such that the decomposition of any vector $b$ in the neighborhood has a unique decomposition $\{v_g^b\}$ minimizing the norm $||b||_{2,1,\mathcal{G}}$.  In particular, the decomposition $\{ v_g^0\}$, minimizing the norm
$||\beta^0||_{2,1,\mathcal{G}}$ is unique. Further, this decomposition is such that $v_g^0 = \textbf{0}$ for all $g \in G_{Ho}$.
\end{asm}

Assumptions~\ref{A:FIXED} and~\ref{B:FIXED} are directly taken from
the grouped lasso setting. Assumption~\ref{correct} is another such
condition adapted to our setting. A direct adaptation would be that 
{\em there exists some $G\subseteq \mathcal{G}$, such that $\cup_{g\in G} g =
\mbox{supp}(\beta^0)$}. This property is
implied by Assumption~\ref{correct}. Note that these three assumptions
are analogous to those needed for the consistency result given
in~\cite{Laurent-graph-lasso}. This assumption also addresses indirectly the issue of identifiability of the groups. For example, for $M=3$, and $\mathcal{G} = \{\{1, 2\}, \{2, 3\}, \{1, 3\}\}$, the target $\beta^0 = (a, a, a)$ does not admit a unique, norm minimizing decomposition within any neighborhood. Similarly, we can create the set $\{1, 2, 3\}$ in four possible ways from unions of members of $\mathcal{G}$. Thus, this particular $\mathcal{G}$ does not satisfy Assumption~\ref{correct} for some targets. 

In the following result, we consider the adaptive overlapping grouped lasso of Equation~\ref{adopt}. We now propose a set of weights $\{\lambda_g\}$ for the adaptive overlapping grouped
lasso.  If we let $\beta^{OLS} =
(\textbf{X}^T\textbf{X})^{-1}\textbf{X}^T\textbf{y}$, and let
$\{v_g^{OLS} \} = \mathcal{V}(\beta^{OLS})$ be any decomposition
minimizing the norm $||\beta^{OLS}||_{2,1,\mathcal{G}}$. Then, let $\lambda_g =
1/||v^{OLS}_g{||}$. This choice of weights gives us our main result:

\begin{thm}
\label{asymProp}
Consider the adaptive overlapping grouped lasso. Suppose Assumptions~\ref{A:FIXED},~\ref{B:FIXED}, and~\ref{correct}
hold.  Let
$\beta^{OLS} = (\textbf{X}^T\textbf{X})^{-1}\textbf{X}^T\textbf{y}$, and let
$\{v_g^{OLS} \} = \mathcal{V}(\beta^{OLS})$ be any decomposition
minimizing the norm $||\beta^{OLS}||_{2,1,\mathcal{G}}$.  Then, let $\lambda_g =
\frac{1}{||v^{OLS}_g{||^\gamma}}$, for $\gamma>0$ such that
$n^{(\gamma + 1)/2}\lambda \to \infty$.  If $\sqrt{n}\lambda \to 0$, then, as $n \to \infty$:
\begin{align}
\sqrt{n}(\widehat{{\beta}} - \beta^0) \to Z.
\end{align}
Where the above is convergence in distribution.  The vector $Z$ has
entries: 
\begin{align}
Z_H &\sim N_{|H|} (0, \sigma^2\mathcal{M}_{H}^{-1}),\\
 Z_{H^c} &= \textbf{0}.
\end{align}
Where $\mathcal{M}_{H}$ is the sub-matrix of $\mathcal{M}$ consisting of the entries with row and column indices in $H$.
\end{thm}

We now make some comments on the result.
\begin{enumerate}
\item {\em In the non-overlapping case, our result reduces
    to previous results from~\cite{nardi-grouped-lasso}.}  In particular, the weights are clearly $\lambda_g = 1 /
||\beta^{OLS}_g ||^\gamma$. 
Given this, we could ask what is the consequence of  simply choosing $\lambda_g = 1 /
||\beta^{OLS}_g ||^\gamma$ for the adaptive weights in any case?  In the
proof of the result, the impact is for the case when $g \in G_{Ho}$.
In summary, the term  $n^{\gamma/2}||\beta^{OLS}_g||^\gamma$ is no longer
$O_p(1)$, since $||\beta^0_g|| >0$.  Then, we get the following
distribution:
\begin{align}
Z_{Ho} &\sim N_{|Ho|} (0, \sigma^2\mathcal{M}_{Ho}^{-1}),\\
 Z_{Ho^c} &= \textbf{0}.
\end{align}
 The resulting distribution is nonzero
with positive probability in coordinates that are zero in
$\beta^0$. In this situation, the problem can be remedied by assuming
that $G_{Ho}$ is empty, that is: 
\begin{asm}
\label{A:SEP}
(Separation of support) $\exists G \subset \mathcal{G}$ such that $\cup_{g \in G} g = H$ and
$\cup_{g \notin G} g = H^c$.
\end{asm}
For many settings with overlap, this is an overly restrictive
assumption.  Note that this assumption corresponds to assuming the groups are
correct in the non-overlapping grouped lasso.  If the groups are incorrect, the
result of this proposition gives us some insight as to what goes wrong
asymptotically.

\item {\em The result gives a consequence of having an ``incorrect''
    set of groups, relative to the support of $\beta^0$.} When the
  condition $\forall \ g \in G_{Ho} \ ; v_g^0 = \textbf{0}$ of
  Assumption~\ref{correct} is violated,   we have that
  $n^{\gamma/2}||\beta^{OLS}_g||^\gamma$ is no longer   $O_p(1)$ for
  $g \in G_{Ho}$, and the consequence is similar to the   previous
  remark.  Again, we get the wrong asymptotic mean, and the
  estimator does not have good selection properties.  Such a violation
  Assumption~\ref{correct} implies that the structure implied by
  $\mathcal{G}$ is not sufficient to capture the structure in
  $\beta^0$.  

\item {\em These results exclude some types of structures: in
    particular nested groups in $\mathcal{G}$}. In particular, the
  uniqueness assumption implies that we can not use a $\mathcal{G}$ which contains nested groups. In
  this case, given
  a set of groups, the uniqueness condition of
  Assumption~\ref{correct} are violated for some $\beta^0$.  For
  example, suppose $p = 5$ and  
\begin{align}
\mathcal{G} = \{ \{1,2\},
  \{3,4\}, \{1,2,3,4\}, \{5\}\}.
\end{align}
Then, for $\beta^0 = [a,a,0,0,c]$,
  then there are an infinite number of decompositions minimizing the
  $||\cdot||_{2,1,\mathcal{G}}$.  In particular, for any $\alpha \in (0,a)$,
   the following decomposition minimizes the norm:
\begin{align}
v_1^0 &= [a-\alpha,a-\alpha,0,0,0],\\
v_2^0 &= [0,0,0,0,0],\\
v_3^0 &= [\alpha,\alpha,0,0,0],\\
v_4^0 &= [0,0,0,0,c].
\end{align}
Then, consider the weights $\lambda_g =
  ||v_G^{OLS}||$.   In almost all data
  applications we have: $\mbox{supp}(\beta^{OLS}) \supset
  \{1,2,3,4\}$.  The minimizing decomposition of
  $||\beta^{OLS}||_{2,1,\mathcal{G}}$ will clearly have $v^{OLS}_{\{1,2\}} =
  v^{OLS}_{\{3,4\}} = 0$.  This effectively excludes the first two
  groups, and we will be unable to detect all possible sparsity
  patterns. More generally, using the same argument as the example, we
  can state that in the case where groups are nested, there exist some
  $\beta^0$ which cannot be uniquely   decomposed to minimize the
  $||\cdot||_{2,1,\mathcal{G}}$ norm.  Thus, using nested groups degrades the asymptotic guarantees of the
  overlapping grouped lasso. This property precludes using a complex nested
  set of groups to encode multiple structures.
\end{enumerate}

\section{Simulation Study}

We now present the results of a simulation study to illuminate and support our earlier theoretical claims. For ease of comparison, we imitate the setting of~\cite{zhang_benefits}. Here, we explore issues most pertinent to the overlapping groups lasso, leaving aside some of the issues addressed by the simulation study in~\cite{zhang_benefits}. We generate an $n \times p$ design matrix \textbf{X} with $\mbox{i.i.d.}$ standard normal entries, with each row scaled so it has unit magnitude. We next generate a structured sparse $\beta^0$ vector with the nonzero entries defined as the union of the first $k$  groups from our set of groups $\mathcal{G}$. We choose the first $k$ groups to achieve a consistent amount of overlap in $\beta^0$ with respect to $\mathcal{G}$ between trials. We define $k$, $\mathcal{G}$, $n$, and $p$ separately in each experiment. After constructing our response from $\textbf{X}$ and $\beta^0$, we add zero mean Gaussian noise with standard deviation $\sigma = 0.01$. We compare the standard lasso against the overlapping groups lasso, with set of groups $\mathcal{G}$. As in~\cite{zhang_benefits}, we adopt the following metric to evaluate the performance of both estimators:
\begin{align}
\mbox{Recovery Error}: \frac{||\beta^0 - \widehat{\beta}||_2}{||\beta^0||_2}
\end{align}
We conduct the following pair of experiments:

\begin{enumerate}
\item {\em Study on the effect of overlap}. Here, we simulate a problem that has nearly constant difficulty for the ordinary, un-grouped, lasso, but increasing difficulty for the grouped lasso.  We set $p=512$, and  set each group so that it consists of $8$ consecutive (by index) predictors. We then vary $\mathcal{G}_{\mbox{Overlap}} \in \{1,2,\ldots,8\}$. For example, with $\mathcal{G}_{\mbox{Overlap}} = 1$, our first two groups are $g_1 = \{1,2,\ldots,8\}; g_2 = \{9,10,\ldots,16\}$, and with with $\mathcal{G}_{\mbox{Overlap}} = 2$, $g_1 = \{1,2,\ldots,8\}; g_2 = \{8,9,\ldots,15\}$, and so forth. We select $k = \mbox{ceiling}((64 - 8)/(8 + \mathcal{G}_{\mbox{Overlap}})) + 1$ groups to be nonzero in $\beta^0$, and set $n = 192$.

\item {\em Study on the effect of sample size}. We adopt a similar setting of the first experiment. We set $\mathcal{G}_{\mbox{Overlap}} = 4$, and set $\mathcal{G}, p, k$ in a similar manner as the first experiment. We consider $n$ satisfying $\log_2(n/48) \in \{0, 1, 2, 3, 4\}$. 

\end{enumerate}

The purpose of the first experiment is to study the effect of increasing complexity of $\mathcal{G}$ on estimation performance. For $\mathcal{G} \in \{1, 2, 3, 4\}$, we see that as the degree of overlap increases, the estimator performance degrades, though not dramatically in these settings. For $\mathcal{G} = 5$, with groups of size $8$, we can see that due to the consecutive placement of the signal, about half of the groups may be dropped without degradation in performance, and we return to the setting and performance of $\mathcal{G} = 1$. For $\mathcal{G} \in \{6, 7, 8\}$, the estimator again does worse than in the case of no overlap, but no worse than $\mathcal{G} = 4$. This result supports the discussion surrounding Assumption~\ref{asmREC} and Theorem~\ref{finiteProp}, but still indicates that the procedure is more robust to overlap than postulated in~\citet{zhang_benefits}.

In the sample size study, we see that for a reasonable $(\mathcal{G}_{\mbox{overlap}} = 4)$ set of groups, the estimator outperforms the lasso: it is able to achieve a limiting level of recovery error for lower sample sizes than the lasso. This supports the conclusions of Theorem~\ref{finiteProp}, as well as the conclusions from the literature about the grouped lasso, e.g.~\citet{zhang_benefits} and~\cite{Lounici_takingadvantage}. We thus see that even in the overlap case, the procedure still enjoys a benefit due to group sparsity.

\begin{figure}
\centering 
\includegraphics[width=4.5in]{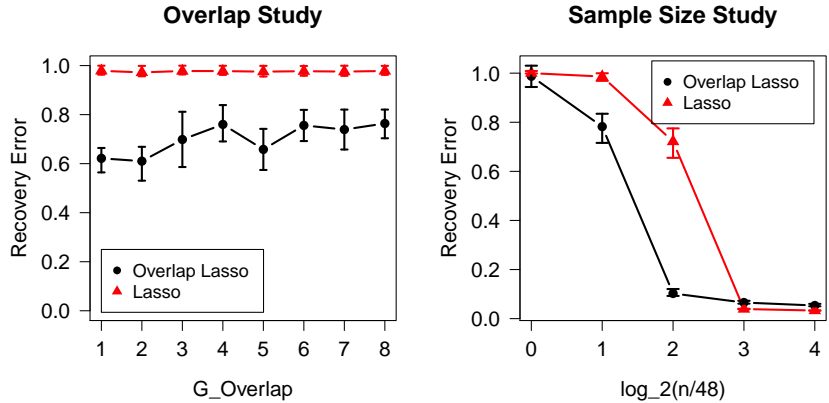}
\caption{Results of the simulation study. We compare the overlapping grouped lasso to the ordinary lasso. Left: study on the degree of overlap of the groups. Right: study on sample size.}
\label{ogl-sim} 
\end{figure}

\section{Discussion and Conclusions}

In the previous two sections, we have given results on the performance of the overlapping grouped lasso in both the finite sample and asymptotic setting. One of the basic steps in practical applications of this procedure is the choice of the collection of groups $\mathcal{G}$. In both cases, we showed that an overly complex choice of $\mathcal{G}$ degrades the theoretical guarantees on the performance of the estimator.  In the case where the dimension of the problem is fixed, increasing the number of groups leads to less tight upper bounds on both prediction and estimation in the finite sample case.  In the asymptotic setting, nested groups lead to inconsistent selection of the true sparsity pattern.   Nonetheless, when $\mathcal{G}$ is suitably chosen, we still see that the procedure retains the theoretical benefits of the grouped lasso demonstrated in previous literature.

In summary, we find that the overlapping grouped lasso is a useful extension of the grouped lasso that must be used with caution.  The flexibility allowed by overlapping groups is valuable in many applications, and can encode a wide variety of structures as collections of groups. We have shown that allowing for overlap does not remove many of the theoretical properties and benefits proven for the lasso and grouped lasso.  However, the procedure must be used with caution. While the flexible nature of the procedure suggests that the analyst may encode many structures simultaneously, this approach is not supported by the results in this paper.

\appendix

\section{Proofs}
\subsection{Finite Sample Result}
\subsubsection{Auxillary Lemmas}
\begin{lem}
\label{chiSquared}
Let $\chi^2_D$ be a chi-squared random variable with $D$ degrees of
freedom.  Then:
\begin{align}
\mathbb{P}(\chi^2_D > D + x) \leq \exp \left( -\frac{1}{8} \min
  \left\{ x,\frac{x^2}{D} \right\}\right).
\end{align}
\end{lem}
\begin{proof}
See Lemma A.1 from~\cite{Lounici_takingadvantage}.
\end{proof}

\begin{lem}
\label{holderExt}
Let $\alpha, \beta \in \mathbb{R}^p$, then, $\forall \mathcal{G}$:
\begin{align}
\alpha^T\beta \leq \mathcal{G}_{\mbox{overlap}}^{1/2}\left( \max_g||\alpha_g||_2  \right) ||\beta||_{2,1,\mathcal{G}}.
\end{align}
\end{lem}
\begin{proof}
Let $\{v^{\beta *}_g\}$ denote any decomposition of $\alpha$ minimizing the norm $||\beta||_{2,1,\mathcal{G}}$. Then, beginning with Holder's inequality the following chain gives the desired result:
\begin{align}
\alpha^T\beta &\leq || \alpha ||_\infty ||\beta||_2\\
&\leq \left|\left| \sum_g \alpha_g \right|\right|_\infty \left|\left| \sum_g v^{\beta *}_g\right|\right|_2\\
&\leq   \mathcal{G}_{\mbox{overlap}}^{1/2}\left( \max_g ||\alpha_g||_2\right) \sum_g  ||v^{\beta *}_g||_2\\
&= \mathcal{G}_{\mbox{overlap}}^{1/2}\left( \max_g||\alpha_g||_2  \right) ||\beta||_{2,1,\mathcal{G}}.
\end{align}
\end{proof}

\label{lemmaFiniteSection}
\begin{lem}
\label{lemmaOracle}
Consider the model in Equation~\ref{mod}.  Suppose $|\mathcal{G}| = M
\geq 2$, and $n \geq 1$.  Assume that the entries of $\epsilon$ are
$\mbox{i.i.d.}$ Gaussian with mean 0 and variance $\sigma^2$.  Let $\textbf{X}$
be normalized so that the the diagonal entries of
$\textbf{X}^T\textbf{X}/n$ are all equal to 1. 
Let $\{v_g^{\widehat{\beta}-\beta}\}$ denote a decomposition of $\widehat{\beta} - \widehat{\beta}$ minimizing the $||\cdot||_{2,1,\mathcal{G}}$ norm.  
Let $J = J(\beta^0) =
\{g : v_g^0 \neq 0\}$ be the set of groups that are
nonzero in the norm minimizing decomposition of $\beta$. Let:
\begin{align}
\lambda = \frac{2\sigma \sqrt{\max_g|g|}}{\sqrt{n}} \left( 1 + \frac{A\log M}{\sqrt{\max_g|g|}}\right)^{1/2}
\end{align}
Here, $A>8$.  Define $q = \min \left(A\sqrt{\min_g |g|}/8,8 \log M \right)$.  Then, with probability at least
$1-M^{1-q}$, for any solution $\widehat{\beta}$ to Equation~\ref{opt}, for
all $\beta \in \mathbb{R}^p$, the following inequality holds:
\begin{align}
\label{l1res1}
\frac{1}{n}||\textbf{X}(\widehat{\beta} - \beta^0)||^2 +
\lambda ||\widehat{\beta} - \beta||_{2,1,\mathcal{G}} &\leq
\frac{1}{n}||\textbf{X}({\beta} - \beta^0)||^2 + 4\lambda\sum_{g \in
  J} ||v_g^{\widehat{\beta}-\beta}||.
\end{align}
\end{lem}
\begin{proof}
We follow the proof strategy of~\cite{Lounici_takingadvantage}. For all $\beta \in \mathbb{R}^p$, we have:
\begin{align}
\frac{1}{n}||X \widehat{\beta} - y||^2 + 2\lambda ||\widehat{\beta}||_{2,1,\mathcal{G}} 
\leq \frac{1}{n}||X {\beta} - y||^2 + 2\lambda ||\beta||_{2,1,\mathcal{G}}
\end{align}
Let $y = X\beta^0 + \epsilon$ to obtain:
\begin{align}
\label{lem1step1}
\frac{1}{n}||X (\widehat{\beta}-\beta^0)||^2 \leq \frac{1}{n}||X
({\beta}-\beta^0)||^2 + \frac{2}{n}\epsilon^TX(\widehat{\beta} - \beta) +
2\lambda \left( ||\beta||_{2,1,\mathcal{G}} - ||\widehat{\beta}||_{2,1,\mathcal{G}}\right)
\end{align}
We now examine the second term on the right hand side:
\begin{align}
\frac{2}{n}\epsilon^TX(\widehat{\beta} - \beta) 
&\leq  \frac{2\mathcal{G}_{\mbox{overlap}}^{1/2}}{n}\left(  \max_g ||\epsilon^TX_g|| \right) ||\widehat{\beta} - \beta||_{2,1,\mathcal{G}}\\
&= \frac{2\mathcal{G}_{\mbox{overlap}}^{1/2}}{n}\left(  \max_g \sqrt{\sum_{j \in g} \left( \sum_{i=1}^n
    X_{ij}\epsilon_i \right)^2 } \right)  ||\widehat{\beta} - \beta||_{2,1,\mathcal{G}}.
\end{align}
Here, we apply our version of H\"{o}lder's inequality (Lemma~\ref{holderExt}).  We now consider the event:
\begin{align}
\mathcal{A} = \left\{  \frac{1}{n} \left(  \max_g \sqrt{\sum_{j \in g}
      \left( \sum_{i=1}^n X_{ij}\epsilon_i \right)^2 } \right) \leq \frac{\lambda}{2\mathcal{G}_{\mbox{overlap}}^{1/2}} \right\}.
\end{align}
Note that random variables $V_{g(j)} = \frac{1}{\sigma \sqrt{n}}
\sum_{i=1}^n X_{ij}\epsilon_i$, where $g(j)$ denotes the $jth$ element
of $g \in \mathcal{G}$, are standard Gaussian random
variables. Within a group, they have a multivariate normal distribution with covariance matrix $\textbf{X}_g^T\textbf{X}_g / (\sigma^2 n)$, where $\textbf{X}_g$ denotes the sub matrix of $\textbf{X}$ consisting of the columns indexed by the group $\textbf{X}_g$. It then follows that, provided $\textbf{X}_g$ admits a Cholesky decomposition, that $(\textbf{X}_g^T\textbf{X}_g)^{-1/2}\textbf{X}_g\boldsymbol{\epsilon} / \sigma^2 n$ is a vector of i.i.d. standard normal random variables. Thus, letting $\rho_g$ denote the maximal absolute eigenvalue of $(\textbf{X}_g^T\textbf{X}_g)^{-1/2}$, we have $||\textbf{X}_g\boldsymbol{\epsilon}/\sigma^2n|| \leq \rho_g ||(\textbf{X}_g^T\textbf{X}_g)^{-1/2}\textbf{X}_g\boldsymbol{\epsilon} / \sigma^2 n||$ by properties of the operator norm of $(\textbf{X}_g^T\textbf{X}_g)^{1/2}$. Now, for any $g \in \mathcal{G}$ define:
\begin{align}
\gamma_g =  \frac{2\sigma\sqrt{|g|\mathcal{G}_{\mbox{overlap}}}}{\sqrt{n}} \left( 1 + \frac{A\log M}{\sqrt{|g|}}\right)^{1/2}.
\end{align}
Note, $\forall g \in \mathcal{G}; \ \gamma_g \leq \lambda$.  Now:
\begin{align}
\mathbb{P}\left( \sum_{j \in g}
      \left( \sum_{i=1}^n X_{ij}\epsilon_i \right)^2 \geq
      \frac{\lambda^2n^2}{4\mathcal{G}_{\mbox{overlap}}}  \right) &\leq \mathbb{P}\left(
      \rho_g^2 \chi^2_{|g|} \geq \frac{\lambda^2n}{4\sigma^2\mathcal{G}_{\mbox{overlap}}}\right)\\
 &\leq \mathbb{P}\left(
      \chi^2_{|g|} \geq \frac{\gamma_g^2n}{4\sigma^2\rho_g^2\mathcal{G}_{\mbox{overlap}}}\right)\\
&= \mathbb{P}\left( \chi^2_{|g|}
 \geq \rho_g^{-2}(|g| + A\sqrt{|g|}\log M)\right)\\
&\leq \exp \left( -\frac{\rho_g^{-2}A \log M}{8} \min\left\{\sqrt{|g|}, A \log
      M\right\} \right)\\
&\leq \exp \left( -\frac{\min_g[\rho_g^{-2}]A \log M}{8} \min\left\{ \left(\min_g \sqrt{|g|}\right), A \log
      M \right\} \right)
\end{align}
In the above, we used Lemma~\ref{chiSquared} for the probability bound on
$\chi^2$ variables. We now apply the union bound to obtain:
\begin{align}
\mathbb{P}(\mathcal{A}^c) &\leq M \exp \left( -\frac{\min_g[\rho_g^{-2}] A \log M}{8}
  \min\left\{ \left(\min_g \sqrt{|g|}\right), A \log
      M \right\} \right)\\
&\leq M^{1-q}
\end{align}
Now, on the event $\mathcal{A}$, we can obtain, from
Equation~\ref{lem1step1}:
\begin{align}
\frac{1}{n}||X (\widehat{\beta}-\beta^0)||^2  &+ \lambda ||\widehat{\beta} - \beta||_{2,1,\mathcal{G}} \leq\\
 \frac{1}{n}||X({\beta}-\beta^0)||^2 &+ 2\lambda \left(||\widehat{\beta} - \beta||_{2,1,\mathcal{G}} + ||\beta||_{2,1,\mathcal{G}} - ||\widehat{\beta}||_{2,1,\mathcal{G}}\right)\\
&\leq \frac{1}{n}||X ({\beta}-\beta^0)||^2 + 4\lambda \sum_{g \in
  J} ||v_g^{\widehat{\beta}-\beta}||
\end{align}
Where $\{v_g^{\widehat{\beta}-\beta}\}$ denotes a decomposition of $\widehat{\beta} - \widehat{\beta}$ minimizing the $||\cdot||_{2,1,\mathcal{G}}$ norm. Note that the last line follows from the fact that $||\cdot||_{2,1,\mathcal{G}}$ obeys the triangle inequality. This gives us the desired result in Equation~\ref{l1res1}.  
\end{proof}

\subsubsection{Proof of Theorem~\ref{finiteProp}}
\label{propositionProofOracle}
Again, we follow the proof strategy of~\cite{Lounici_takingadvantage}.
Fix a decomposition of $\beta^0$: $\{v_g^0\}$.  Let $J = J(\beta^0) =
\{g : v_g^0 \neq 0\}$.  Let the event $\mathcal{A}$ in
Lemma~\ref{lemmaOracle} hold and let $\beta = \beta^0$ in the
inequality~\ref{l1res1}:
\begin{align}
\lambda||\widehat{\beta} - \beta^0||_{2,1,\mathcal{G}} &\leq
4\lambda\sum_{g \in
  J} || v^{\widehat{\beta}-\beta^0}||,\label{startOfTwo}\\
\implies \sum_{g \in  J^c} || v^{\widehat{\beta}-\beta^0} || &\leq  3\sum_{g \in
  J} || v^{\widehat{\beta}-\beta^0} ||.
\end{align}
Thus, we can apply Assumption~\ref{asmREC} with $\Delta = (\widehat{\beta}
- \beta^0), \mathcal{V}(\Delta) = \{v^{\widehat{\beta}-\beta^0} \}$ to obtain:
\begin{align}
\sum_{g \in J}\left|\left|v^{\widehat{\beta}-\beta^0}\right|\right| \leq \frac{||X(\widehat{\beta} - \beta^0)||}{\kappa \sqrt{n}}
\end{align}
Again, when the event $\mathcal{A}$ in
Lemma~\ref{lemmaOracle} hold and for $\beta = \beta^0$ in the
inequality~\ref{l1res1}:
\begin{align}
\frac{1}{n}||\textbf{X}(\widehat{\beta} - \beta^0)||  &\leq 
4\lambda  \sum_{g \in J}|| v^{\widehat{\beta}-\beta^0}||\\
&\leq \frac{4\lambda}{\kappa \sqrt{n}}||X(\widehat{\beta} -
\beta^0)||\\
\implies \frac{1}{n^2}||\textbf{X}(\widehat{\beta} - \beta^0)||^2 &\leq
\frac{16\lambda^2 }{\kappa^2 n}\\
\implies \frac{1}{n}||\textbf{X}(\widehat{\beta} - \beta^0)||^2 &\leq
\frac{64\sigma^2  }{\kappa^2 n} \left(\max_g|g|
 + A\log M \right)
\end{align}
This corresponds to the result in Equation~\ref{p1res1}. Equation~\ref{p1res2} follows from an analogous chain as the above, beginning with the inequality~\ref{startOfTwo}.


\subsection{Asymptotic Setting}

Before we prove the main result, we give the following lemma.

\begin{lem}
\label{oplemma}
Let Assumption~\ref{correct} hold.  For $g \in G_{Ho}$ $n^{\gamma/2}(||v_g^{OLS}||)^\gamma$ is $O_p(1)$ for $\gamma >0$.
\end{lem}
\begin{proof}
By Assumption~\ref{correct}, we may denote $\{v_g^{0} \} =
\mathcal{V}(\beta^{0})$ as the unique decomposition minimizing the norm
$||\beta^{0}||_{2,1,\mathcal{G}}$. 
To make the dependence on $n$ explicit, we denote $\beta^{OLS}_n$ as
the ordinary least squares estimate for $\beta^0$ using $n$ data
points. We know $\beta^{OLS}_n \to \beta^o$ in probability,
as $n \to \infty$. By Assumption~\ref{correct}, there exists an $N$ such that, with high probability, $\beta^{OLS}_n$ has a unique decomposition for all $n \geq N$. We denote this unique decomposition as: $\{v_g^{OLS,n} \} =
\mathcal{V}(\beta^{OLS}_n)$, minimizing
$||\beta^{OLS}_n||_{2,1,\mathcal{G}}$.

We next write $\beta^{OLS}_n = \beta^0 + \delta_n$, and then define the decomposition $v_g^{\delta_n} = v_g^0 - v_g^{OLS_n}$. Recall that for $g \in G_{H_o}$, we have $||v_g^0|| = 0$ and  furthermore $||\beta^{OLS}_n||_{2,1,\mathcal{G}} \to ||\beta^{0}||_{2,1,\mathcal{G}}$ in probability. Thus, considering the terms in $||\beta^{OLS}_n||_{2,1,\mathcal{G}}$ corresponding to those $g \in G_{H_o}$ we conclude $||v_g^{\delta_n}|| \to 0$ in probability as $n \to \infty$ for $g \in G_{H_o}$.   

Finally, for $g \in G_{H_o}$: $\sqrt{n}(||v_g^{OLS}||) = \sqrt{n}(||v_g^{OLS}|| - ||v_g^0||) = \sqrt{n}(||v_g^{\delta_n}|| - 0) \in O_p(1)$. The result then follows for $\gamma > 0$ by the continuous mapping theorem.
\end{proof}

\subsubsection{Proof of Theorem~\ref{asymProp}}

We follow the general proof strategy of Theorem 3.2
from~\cite{nardi-grouped-lasso}, which is adapted from similar results
on the lasso from~\cite{fuKnightLasso} and~\cite{Zou_2006}.
First, define $\beta_n = \beta^0 + \frac{u}{\sqrt{n}}$.  Let
$\{v_g^0\} = \mathcal{V}(\beta^0); \{v_g^n\} = \mathcal{V}(\beta^n)$
be decompositions of $\beta^0$ minimizing
$||\beta^0||_{2,1,\mathcal{G}}$, and $||\beta_n||_{2,1,\mathcal{G}}$,
respectively.  Therefore, the following is a decomposition of $u$:
$\forall \ g \in \mathcal{G}$, $v_g^u
= \sqrt{n}(v_g^n - v_g^0)$.

To begin, we write the objective from Equation~\ref{adopt} (multiplied by $\frac{n}{2}$)
as:
$$
Q_n(u) = \frac{1}{2} \left| \left| \frac{1}{\sqrt{n}} X u + \epsilon
  \right| \right|^2 + \sum_g n \lambda \lambda_n \left|\left| v^0_g +
\frac{1}{\sqrt{n}}v_g^u\right|\right|
$$
Let:
\begin{align*}
D_n(u) &= Q_n(u) - Q_n(0)\\
&= \left( \frac{1}{2n} u^T X^TX u - \frac{1}{\sqrt{n}} u^T X
  \epsilon\right)\\
& \ + \sqrt{n} \lambda \sum_g \lambda_g \sqrt{n} \left( \left|\left| v^0_g +
  \frac{1}{\sqrt{n}} v^u_g \right|\right|  - || v_g^0||\right)\\
&= I_{1,n} + \sum_g I_{2,n,g}
\end{align*}
We now proceed to examine the terms in the second summation.  The
behavior of these terms depends on the group $g$:
\begin{itemize}
\item For $g \in G_H$, we have $\lambda_n \to 1 / ||v_g^o ||_2^\gamma$ in
probability, by the uniqueness of the decomposition $\{v_g^0\}$ along with Assumption~\ref{correct}.  Also:
$$
\sqrt{n} \left( \left|\left|v^0_g +
  \frac{1}{\sqrt{n}} v^u_g \right|\right|  - || v_g^0||\right)
\to \frac{(v^u_g)^Tv^0_g}{||v_g^0||}.
$$
Since $\sqrt{n}\lambda = o(1)$, then the term $I_{2,n,g} \to 0$.
\item For $g \in G_{Hc}$, $n^{\gamma/2}||v^{OLS}_g||^\gamma =
O_p(1)$ and:
$$
\sqrt{n} \left( \left|\left|v^0_g +
  \frac{1}{\sqrt{n}} v^u_g \right|\right|  - || v_g^0||\right)
= ||v^u_g||.
$$
Since, $n^{(\gamma+1)/2}\lambda \to \infty$, then $I_{2,n,g} \to
\infty$.
\item For $g \in G_{Ho}$, and
  $n^{\gamma/2}||v^{OLS}_g||^\gamma_2$ is $O_p(1)$ by Lemma~\ref{oplemma}.
  As before,
$$
\sqrt{n} \left( \left|\left|v^0_g +
  \frac{1}{\sqrt{n}} v^u_g \right|\right|  - || v_g^0||\right)
= ||v^u_g||.
$$
So $I_{2,n,g} \to \infty$.
\end{itemize}

Now, $I_{1,n} \to \frac{1}{2}u^T\mathcal{M}u - u^TW$, where $W \sim
  N_p(0,\sigma^2\mathcal{M})$.  Since $p$ is fixed and finite, then it follows
  that $D_n(u) \to D(u)$, where:
\begin{align*}
D(u) = \left\{
     \begin{array}{lr}
       \frac{1}{2}u^T\mathcal{M}u - u^TW  & \mbox{if }\forall g \notin G_{H}: v^u_{g}= 0\\
       \infty  & \mbox{else}
     \end{array}
   \right.
\end{align*}
Now, $u = (\mathcal{M}_{H}^{-1}W,0)^T$ minimizes $D(u)$ and so by
the argmax theorem from~\cite[Corollary 3.2.3]{vandervaartwellner}, the result follows.

\bibliography{refs}
\bibliographystyle{imsart-nameyear}

\end{document}